\documentclass[twocolumn]{article}

\def\abstract{\begin{center}\vspace{-0.8em}\large\bf{Abstract}\end{center}\quotation\small}

\usepackage{titlesec}
\titleformat*{\section}{\large\bfseries}
\titleformat*{\subsection}{\normalsize\bfseries}
\usepackage{graphicx}
\usepackage{amsmath}
\usepackage{amssymb}
\usepackage{amsthm}
\usepackage{bm}
\usepackage{algorithm}
\usepackage{algpseudocode}
\usepackage{url}
\usepackage{authblk}
\usepackage{fancyhdr}

\usepackage[round]{natbib}

\newtheorem{theorem}{Theorem}
\newtheorem{lemma}{Lemma}

\newtheorem{corollary}{Corollary}
\newtheorem{remark}{Remark}

\newcommand*{\Ex}[2]{\mathop{\mathbb{E}}_{#1}{\left[#2\right]}}
\newcommand*{\KL}[2]{D_{KL}\left[#1 || #2\right]}
\newcommand{\cL}{\mathcal{L}}
\newcommand{\cD}{\mathcal{D}}

\newcommand{\cR}{\mathcal{R}}
\newcommand{\cF}{\mathcal{F}}
\newcommand{\cX}{\mathcal{X}}
\newcommand{\cY}{\mathcal{Y}}
\newcommand{\argmax}{\mathop{\rm arg~max}\limits}

\newcommand{\T}{^{\rm T}}
\algnewcommand{\IIf}[1]{\State\algorithmicif\ #1\ \algorithmicthen}
\algnewcommand{\EndIIf}{\unskip\ \algorithmicend\ \algorithmicif}

\allowdisplaybreaks
\begin{document}

\title{Stopping criterion for active learning based on deterministic generalization bounds}

\author[1]{Hideaki Ishibashi}
\author[2]{Hideitsu Hino }

\affil[1]{Kyushu Institute of Technology}
\affil[2]{The Institute of Statistical Mathematics/RIKEN AIP}
\date{}

\maketitle

\thispagestyle{fancy}

\begin{abstract}
Active learning is a framework in which the learning machine can select the samples to be used for training. This technique is promising, particularly when the cost of data acquisition and labeling is high. In active learning, determining the timing at which learning should be stopped is a critical issue. In this study, we propose a criterion for automatically stopping active learning. The proposed stopping criterion is based on the difference in the expected generalization errors and hypothesis testing. We derive a novel upper bound for the difference in expected generalization errors before and after obtaining a new training datum based on PAC-Bayesian theory. Unlike ordinary PAC-Bayesian bounds, though, the proposed bound is deterministic; hence, there is no uncontrollable trade-off between the confidence and tightness of the inequality. We combine the upper bound with a statistical test to derive a stopping criterion for active learning. We demonstrate the effectiveness of the proposed method via experiments with both artificial and real datasets.
\end{abstract}

\section{Introduction}
In supervised learning problems, increasing the number of training samples can improve prediction performance. However, to train a predictor, annotated datasets are required, and annotation frequently requires the knowledge of experts or the conduction of experiments with a high cost, such as large-scale experiments or long-term experiments, for example, agricultural examinations. Active learning (AL)~\citep{Settles2009} is a framework in which learners can select data that improve the prediction accuracy. Although various methods have been proposed for selecting new data, only a few studies have reported on the criteria for stopping learning~\citep{Vlachos2008,Ertekin2007,Paisley2010,Krause2007}, which is a critical aspect to make active learning practical. Furthermore, most existing criteria depend on a specific learning tasks, and parameters such as the threshold need to be appropriately defined.

The purpose of this study is to develop a versatile stopping criterion for active learning, that is, a criterion independent of task or loss function. Realization of this goal faces two challenges. The first issue is determining the measure to use for the stopping criterion. The simplest approach is to evaluate the convergence of learning by monitoring the generalization error of the prediction model using the test dataset. However, active learning is often applied in circumstances in which obtaining a sufficient amount of test data is unreasonable. Therefore, it is necessary to evaluate the convergence of learning without using the test dataset. The second issue is the development of a concrete algorithm to determine the stop timing of learning. The simplest approach for this is to set a threshold and stop learning when a value defined based on a certain measure exceeds the threshold.
However, in general, the appropriate threshold may not be known in advance.

In this study, the generalization error is evaluated without using the test data by employing the upper bound of the difference in the expected generalization errors based on the PAC-Bayesian framework~\citep{McAllester1999,McAllester2003,Langford2005,Catoni2007}. Furthermore, the sequence of the difference in the expected generalization errors is regarded as time series data, and the stop timing of learning is determined automatically via a runs test~\citep{Wald1940}.

The major contributions of our work are as follows:

A versatile stopping criterion for active learning is proposed. The proposed method is applicable to arbitrary cost functions and can be applied to both classification and regression tasks. Moreover, the proposed criterion can evaluate performance of a model at low calculation cost. In the stopping criterion of AL, we need to consider a trade-off between the labeling cost, computational cost for both learning the model and determining the criterion as well as the predictive performance of the model. The proposed method assume the cost of labeling is dominant, and the predictive performance of the model is of importance. The computational cost for evaluating the stopping criterion is of the constant order, hence it does not increase the over all computational cost.

Theoretical guarantees for the criterion are derived based on the PAC-Bayesian theory. The proposed criterion is the upper bound of the optimal value in terms of the PAC-Bayesian framework.
It is notable that even though the proposed criterion is derived from the PAC-Bayesian theory, it is a deterministic bound; therefore, there is no trade-off between the confidence and tightness of the bound. We combine the proposed bound with statistical test to realize a statistically sound criterion for stopping active learning.

\section{Active learning and its stopping criteria}
Let $x \in \cX$ and $y \in \cY$ be the input variable and the corresponding output variable, respectively, in certain domains $\mathcal{X}$ and $\mathcal{Y}$. Supervised learning considers the problem of estimating the predictor $f: x \mapsto \mathbb{E}[y|x]$. For problems with high annotation cost or acquisition cost, active learning~\citep{Settles2009} is a practical method for constructing a useful prediction model with the minimum number of annotations or labeling for the output variable. Specifically, active learning yields a new datum by repeating two processes: (i) estimating the predictor from the acquired training data and (ii) determining the new input datum $x^{\ast}$ to maximize the {\it{acquisition function}} $a(x|f)$ based on the current predictor:
\begin{equation}
x^\ast = \argmax_{x \in \cX}a(x | f).
\label{acquisition_function}
\end{equation}

In practical application of active learning algorithms, the timing to terminate learning is decided based on a predetermined budget, which is called the ``fixed-budget" approach. However, there are two problems in this standard approach. (1) Even within the fixed budget, it is possible that the learner has enough training data, and the fixed-budget approach oversamples in this situation, which misses possible saving of budget. In addition, if we knew that the learner has much room for improvement when we have reached the budget limit, we can claim that we should spend more of the budget. (2) Active learning can be used for outlier removal~\citep{Kobayashi2012,Cohan2018}, but the fixed-budget approach cannot rule out sampling outliers. Developing a stopping criterion for active learning is in these sense important.

The existing stopping criteria for active learning can be classified into accuracy- and uncertainty-based approaches. Methods pertaining to the accuracy-based approach evaluate predictive errors. A typical method is to evaluate the predictive error by using selected unlabeled data or pool data~\citep{Zhu2007,Zhu2008,Zhu2008b,Laws2008}. Another popular method is based on the stability of agreement of multiple predictors~\citep{Bloodgood2009,Bloodgood2013,Altschuler2019,Olsson2009}.
Methods pertaining to the uncertainty-based approach evaluate the uncertainty of prediction by using the pooled data. The margin of support vector machine (SVM) classifier is used for measuring the uncertainty~\citep{Schohn2000,Vlachos2008}. Criteria based on the convergence of the margin or other related quantities evaluated by using the pooled dataset are proposed in ~\citep{Laws2008,Krause2007}. However, most of those criteria depend on the learning models, acquisition function, and problem settings such as the use of classification or regression. 

Apart from the studies on active learning, some other studies have considered the optimal stopping timing for learning algorithms. In the nonparametric regression and neural networks literature, {\it{early stopping}}~\citep{Prechelt2012,Raskutti2011} and its variants~\citep{Wang2018,Raskutti2011} have been widely used in practice to reduce computational time and overfitting. In the framework of Bayesian optimization, several heuristics have been devised~\citep{Lorenz2015,Desautels2014}, although a theoretically supported method with practical utility is yet to be developed. In the multi-armed bandit literature, the best arm identification has been considered as a problem of finding the best model in the minimum number of trials~\citep{Kaufmann2016,Even-Dar2006,Audibert2010,Aziz2018}. Determining an appropriate stopping timing is also a critical issue when running Markov chain Monte Carlo (MCMC) algorithms, in which the sequential fixed-width confidence interval is one of the most popular methods~\citep{chow1965,Gong2016,Galin2006}. However, the methods developed in the studies related to MCMC algorithms focus on the convergence to a stationary distribution, and thus they cannot be directly applied in active learning because the aim of active learning is not to obtain a certain probability model but to train an accurate predictor.

In this work, we focus on a stopping criterion that monitors the difference in the generalization errors before and after a new training sample is obtained. If sufficient test data are available, it is easy to evaluate the difference in the generalization errors. However, this is not the case when active learning is involved. To alleviate this problem, we adopt the PAC-Bayesian framework~\citep{McAllester1999,McAllester2003,Langford2005,Catoni2007,Germain2016}, in which the generalization error is bounded by using only the training dataset.

\section{Evaluation of difference between generalization errors}
Let $S=(X,Y)=\{(x_i,y_i)\}^t_{i=1}, \; (x_i,y_i) \in \cX \times \cY$ be the observed dataset. We assume that $(X,Y)$ is generated by a probability distribution $\cD$. Let $l:\cF\times\cX\times\cY \rightarrow [a,b]$ be a loss function, where $\cF$ is a set of predictors and $0 \leq a < b < \infty$. In the Bayesian framework, we assume that the predictor $f$ is a function-valued random variable and consider its prior and posterior distributions.
The expected risk $\cL_\cD(f)$ and empirical risk $\cL_S(f)$ can be defined as follows:
\begin{equation}
    \cL_\cD(f) = \Ex{\cD}{l(f,x,y)}, \quad
    \cL_S(f) = \frac{1}{t}\sum^t_{i=1}l(f,x_i,y_i). \nonumber
\end{equation}
PAC-Bayesian theory binds the expected generalization error to any posterior $q(f)$ of predictor $f$~\citep{McAllester2003,Catoni2007}, that is, $\Ex{q(f)}{\cL_\cD(f)}$. So far, various upper bounds in classification have been proposed~\citep{McAllester1999,McAllester2003,Langford2005,Catoni2007}. Recently, the upper bound of the expected generalization error in the regression problem has also been proposed~\citep{Alquier2016,Germain2016}.

From the viewpoint of active learning, the greatest advantage of the PAC bound is its universality because it is derived without assuming any specific loss function. Based on the PAC-Bayesian approach, not only different learning models but also both classification and regression problems can be addressed using the unified framework.

\subsection{Upper bound for difference between generalization errors}
Let $q(f|S)$ be the posterior distribution of predictor $f \in \cF$ given a dataset $S=(X,Y)$. The posterior is calculated by using Bayes' theorem as follows: Assuming that the loss function is the negative log-likelihood~\citep{Banerjee2006}, that is,
$l(f,x,y) = -\log{p(y|f)}$, we have $p(y|f,x) = e^{-l(f,x,y)}.$
If a vector of function value $f$ for input $X$ is denoted by $\mathbf{f}_X$, the posterior can be obtained using
\begin{align}
    q(\mathbf{f}_X|S) &= p(Y|\mathbf{f}_X,X)p(\mathbf{f}_X)/ p(Y|X) \\
    &= e^{-t\cL_S(f)}p(\mathbf{f}_X)/p(Y|X),
\end{align}
where $p(\mathbf{f}_X)$ is a prior.
Let $X_\ast$ and $\mathbf{f}_{X_\ast}$ be the test input and its corresponding vector of the function value, respectively. Then, the predictive distribution $q(\mathbf{f}_{X_\ast}|S)$ is given by $q(\mathbf{f}_{X_\ast}|S) = \int{p(\mathbf{f}_{X_\ast}|\mathbf{f}_X)q(\mathbf{f}_X|S)d\mathbf{f}_X}$, where $p(\mathbf{f}_{X_\ast}|\mathbf{f}_X)$ is the conditional probability w.r.t. the prior distribution $p(\mathbf{f}_{X_\ast},\mathbf{f}_X)$.
The difference between the expected generalization error w.r.t $q(f|S)$ and that w.r.t $q(f|S')$ is denoted by $\cR(q(f|S),q(f|S'))$, that is,
\begin{equation}
    \cR(q(f|S),q(f|S')) = \Ex{q(f|S)}{\cL_\cD(f)} - \Ex{q(f|S')}{\cL_\cD(f)}. \nonumber
\end{equation}
Note that the sample sizes of $S$ and $S'$ can be different.

The quantity $\cR(q(f|S_t),q(f|S_{t+1}))$ represents the reduction of the expected generalization error by the data acquisition. Substituting $p(f)$ to $q(f|S_0)$, we have $\cR(p(f),q(f|S_t)) = \sum^t_{i=1}\cR(q(f|S_{i-1}),q(f|S_i))$. This represents the cumulative reduction of the expected generalization error by adding $t$ samples. Because $\Ex{p(f)}{\cL_\cD(f)}$ is constant for added samples, we can assume it is a constant $d$. Then, $\cR(p(f),q(f|S_t))=d-\Ex{q(f|S_t)}{\cL_\cD(f)}$; hence, the convergence of $\cR(p(f),q(f|S_t))$ is equivalent to the convergence of $\Ex{q(f|S_t)}{\cL_\cD(f)}$, which is the rationale behind the definition of $\cR$. We evaluate the convergence of $\cR(q(f|S_t),q(f|S_{t+1}))$ in the framework of a statistical test.

The KL divergence between $p(f)$ and $q(f)$ can be defined as $\KL{p(f)}{q(f)}=\Ex{p(f)}{\log{\frac{dp(f)}{dq(f)}}}$.
Then, the following theorem holds:
\begin{theorem}
Let $q(f|S)$ and $q(f|S')$ be the posteriors w.r.t. predictor $f \in \cF$ given $S$ and $S'$. For any measurable function $\cL_\cD(f)$\footnote{With a probability space $(\cF,\Sigma,q(\cdot|S))$ and measurable space $([a,b],T)$, suppose $\cL_\cD:\cF \rightarrow [a,b]$ is a measurable function. Then, $\forall E \in T$, and we have $\cL^{-1}_\cD(E)\in\Sigma$; hence, there exists a probability distribution $p$ such that $p(E) = q(\cL^{-1}_\cD(E)|S)$, where $\cL^{-1}_\cD$ is the inverse correspondence. Now, we have $\Ex{q(f|S)}{e^{\cL_\cD(f)}}=\Ex{p(\cL)}{e^{\cL}}$, and we can use Jensen's inequality in the proofs of Lemma~1 and Theorem~\ref{theorem1}.}, the following inequality holds:
\begin{equation}
    \cR(q(f|S),q(f|S')) \leq \KL{q(f|S)}{q(f|S')} + C,
    \label{upper_bound}
\end{equation}
\label{theorem1}
where $C= 2\log{\frac{e^a+e^b}{2}}-a-b$. We denote the upper bound by $\tilde{\cR}(q(f|S),q(f|S'))$.
\end{theorem}
If the two posterior probabilities $q(f|S)$ and $q(f|S')$ share a common prior distribution, the following corollary holds:
\begin{corollary}
Let $q(f|S)$ and $q(f|S')$ be the posteriors of $f$ given $S$ and $S'$, respectively. We assume that a prior of $q(f|S)$ and that of $q(f|S')$ are the same probability distribution. Then, for any measurable function $\cL_\cD(f)$, the following inequality holds:
\begin{equation}
    \cR(q(f|S),q(f|S')) \leq \KL{q(\mathbf{f}_{X_+}|S))}{q(\mathbf{f}_{X_+}|S'))} + C. 
\end{equation}
Here, $\mathbf{f}_{X_+}$ is a random variable vector of the function values for inputs $X_+:=X\cup X'$, and the posterior distributions of $\mathbf{f}_{X_+}$ given $S$ and $S'$ are denoted by $q(\mathbf{f}_{X_+}|S)$ and $q(\mathbf{f}_{X_+}|S')$, respectively.
\label{corollary1}
\end{corollary}
Corollary~\ref{corollary1} implies that $\cR(q(f|S),q(f|S'))$ can be evaluated by the upper bound, which is computable by using the observed data.

\subsection{KL divergence between Gaussian processes}
As a specific example, we consider active learning with the Gaussian process~(GP) as the prediction function. We note that the proposed method for stopping active learning is applicable as long as we can estimate the KL divergence between posterior distributions. In the GP, the loss function is assumed to be the negative log likelihood of the Gaussian distribution, and the prior distribution can be obtained as
\begin{equation}
  \left[
    \begin{array}{c}
      \mathbf{y}  \\
      f(x)
    \end{array}
  \right]
  \sim \mathcal{N}\left(
  \left[
    \begin{array}{c}
      \bm{\mu}  \\
      \mu(x)
\end{array}
  \right]
  ,
  \left[
    \begin{array}{cc}
      \mathbf{K}+\beta^{-1}\mathbf{I} & \mathbf{k}(x) \\
      \mathbf{k}\T(x) & k(x,x)
    \end{array}
  \right]\right). \nonumber
\end{equation}
Let $S_t =\{(x_i,y_i)\}^t_{i=1}$ be the observed dataset of size $t$. Then, the posterior can be defined as
\begin{align}
    f(x) | x, S_t &\sim \mathcal{N}(\mu(x),\sigma(x,x)),
\end{align}
where $\mu(x) = \mathbf{k}\T(x)(\mathbf{K}+\beta^{-1}\mathbf{I})^{-1}\mathbf{y}$, $ \sigma(x,x) = k(x,x) - \mathbf{k}\T(x)(\mathbf{K}+\beta^{-1}\mathbf{I})^{-1}\mathbf{k}(x)$, and $\mathbf{y}=(y_1,y_2,\cdots,y_t)\in \mathbb{R}^{t}$.

Let $q(f|S_t)$ and $q(f|S_{t+1})$ be the posteriors of $f$ given $S_t=\{(x_i,y_i)\}^t_{i=1}$ and $S_{t+1}=\{(x_i,y_i)\}^{t+1}_{i=1}$, respectively.
We assume that the prior of $q(f|S_t)$ is the same as that of $q(f|S_{t+1})$. Then, it is easy to calculate $\cR(q(f|S_t),q(f|S_{t+1}))$ in the case of a GP posterior. Let $\mu_t$ and $\sigma_t$ be the mean and covariance functions of $q(f|S_{t+1})$, respectively.
Then, from Corollary~\ref{corollary1}, the following equality holds:
\begin{align}
    &\tilde{\cR}(q(f|S_t),q(f|S_{t+1})) \nonumber \\ =&\KL{\mathcal{N}(\mathbf{f}_{X_{t+1}}|\mathbf{\mu}_t,\mathbf{\Sigma}_t)}{\mathcal{N}(\mathbf{f}_{X_{t+1}}|\mathbf{\mu}_{t+1},\mathbf{\Sigma}_{t+1})}
    + C \nonumber \\
    =&\frac{1}{2}\beta \sigma_t(x_{t+1},x_{t+1})-\frac{1}{2}\log{|1+\beta \sigma_t(x_{t+1},x_{t+1})|} \nonumber \\
    &+\frac{1}{2}\frac{\beta \sigma_t(x_{t+1},x_{t+1})}{\sigma_t(x_{t+1},x_{t+1})+\beta^{-1}}(y_{t+1}-\mu_t(x_{t+1}))^2+C.
    \label{proposed_upper_bound}
\end{align}
Details of the derivation is described in the supplementary material.

\subsection{Conventional PAC-Bayesian bounds}
In this section, we present the typical PAC-Bayesian bounds derived in the literature. The most famous bound is McAllester's bound~\citep{McAllester1999}, which binds the expectation of the generalization error by the training dataset $S=\{(x_i,y_i)\}^t_{i=1}$ as follows:
\begin{align}
    &\mathop{Pr}_{S \sim D}\left(\forall{q}: \Ex{q(f)}{\cL_\cD(f)}\leq \Ex{q(f)}{\cL_S(f)}\right. \\ \notag
     &+\left.\sqrt{
    \left(\KL{q(f)}{p(f)}+\log{(2\sqrt{t}/\delta)} \right)/2t}\right)\geq 1-\delta,
\end{align}
where $\delta \in (0,1]$ is the confidence parameter, and $p(f)$ and $q(f)$ are the prior and posterior distributions, respectively. McAllester's bound is only applicable to classification problems.

For regression problems, \citet{Alquier2016} derived the following bound:
\begin{align}
    \label{Alquier's bound}
    &\mathop{Pr}_{S \sim D}\left(\forall{q}: \Ex{q(f)}{\cL_\cD(f)}\leq \Ex{q(f)}{\cL_S(f)} \right.\\ \notag
    &\left.+t^{-1} \left(\KL{q(f)}{p(f)}-\log{\delta}\right)+1/2(b-a)^2\right)\geq 1-\delta,
\end{align}
in which the range of loss function $l$ is restricted to $[a, b]$. \citet{Germain2016} proposed another bound that does not have any restrictions on the range of $l$; however, this bound is derived by assuming a specific form of the loss function.

The most notable difference between the existing and proposed bounds is that the proposed bound is a {\it{deterministic}} bound. We can guarantee a gap between the posterior distributions before and after adding a new sample without any confidence parameter. This reliability is a particularly important characteristic when the bound is to be applied to a measure of the stopping criterion.

We propose to determine whether to stop the learning or not by testing the convergence of sequence of $\cR(q(f|S_t),q(f|S_{t+1}))$. Our aim is to develop a reliable criterion for stopping active learning in the framework of a statistical test. Ordinary PAC-Bayesian bounds have parameters $\delta$ (similar to $(\alpha,1-\beta)$, the significance level and power, for a statistical test). Since the confidence parameter appears both inside and outside of the probability function, it is not straightforward to cast the PAC-Bayesian bound in a standard statistical test framework. In contrast, using the proposed novel deterministic bound makes it possible to develop a tractable statistical test, as introduced in the next section.

\section{Convergence test}
\subsection{Wald--Wolfowitz runs test}
The Wald--Wolfowitz runs test~\citep{Wald1940} is a nonparametric test of the randomness hypothesis of a given binary sequence. It was originally proposed as a one-sample test but has been extended to two-sample tests by~\citet{Barton1957}. Usually, the null hypothesis is set to be $H_0:p(E_1,E_2,\cdots,E_T)=\Pi^T_{t=1}p(E_t)$, and under this hypothesis, we assume that the data at time $t$ is $E_t \in \{0,1\}$ and the probability that sequence $(E_1,E_2,\cdots,E_T)$ is generated is $p(E_1,E_2,\cdots,E_T)$. It is known that the power of the runs test is superior to that of the Kolmogorov--Smirnov test when the difference in location between the two sequences is small with a large difference in the variance~\citep{Magel1997}.

In the runs test, the sequence of the same number (zero or one) is called the run, and the length of runs are treated as random variables. We denote the random variable for the total number of runs as $U$. Let $t_0$ and $t_1$ be the numbers of zeros and ones, respectively, and let $T$ be the length of the sequence, that is, $T=t_0+t_1$. We assume that $t_1 \geq t_0$. Then, the probability distribution of $U$ under the null hypothesis is as follows:
\begin{align*}
    &p(U=2t) = 2({}_{t_0-1} C_{t-1} {}_{t_1-1} C_{t-1})/{}_T C_{t_0}, \\
    &p(U=2t+1) \\
    =&({}_{t_0-1} C_{t-1} {}_{t_1-1} C_{t-2}+{}_{t_0-1} C_{t-2}  {}_{t_1-1} C_{t-1})/{}_{T} C_{t_0},
\end{align*}
where $t=1,2,\cdots,t_0$. 
$p(U)$ is shown to be a normal distribution with average $\mu =1 + (2t_0t_1)/T$ and variance $\sigma^2= 2t_0t_1(2t_0t_1-T)/(T^2(T-1))$.
Then, the randomness of the sequential data is tested by using the test statistic $Z=(U-\mu)/\sigma$.

\begin{remark}
It is known that the Wald--Wolfowitz test is reasonably powerful when the alternative hypothesis has a Markov property~\citep{David1947}. Unfortunately, in our setting, this is not the case. In future work, we will investigate the condition in which the proposed test is the most powerful.
\end{remark}

\subsection{Proposed method for stopping active learning}
We describe a specific algorithm for stopping active learning with a GP. The proposed stopping criterion is also applicable to other active learning frameworks; however, here we simply select a new input datum based on the uncertainty sampling strategy. Various measures of uncertainty exist, and entropy is one of the reasonable measures~\citep{Settles2009}. Because we consider the GP as a predictor, selecting a new input datum with the maximum entropy is equivalent to selecting the point with the maximum variance:
\begin{equation}
    x^\ast=\argmax_{x}\sigma(x,x),
     \label{acquisition_function_gp}
\end{equation}
where $\sigma(x,x)$ is the covariance function of $q(f|S_{t+1})$, which serves as an acquisition function in Eq.~\eqref{acquisition_function}. Let $R$ be the sequence of $\tilde{\cR}(q(f|S_t),q(f|S_{t+1}))$, i.e., $R=\{r_1,r_2,\cdots,r_T\}$ and $r_t=\tilde{\cR}(q(f|S_t),q(f|S_{t+1}))$. Because $R$ is a sequence of continuous values, we cannot directly perform the runs test for $R$. In this work, following the work by~\citet{Jani2014}, each $r_t$ is converted to $1$ if $r_t \geq \mbox{median}(R)$ and $0$ if $r_t < \mbox{median}(R)$. Algorithm~\ref{alg1} summarizes the procedure explained above.

Finally, we consider the computational aspect. The proposed criterion requires computation of the mean function and covariance function of the posterior $q(f|S_t)$ for evaluating $\tilde{\cR}(q_t,q_{t+1})$. However, the computation cost for evaluating $\tilde{\cR}(q_t,q_{t+1})$ is $\mathcal{O}(1)$ since the mean function and covariance function are already calculated during exploration of new datum. Therefore, the overall computational cost is equal to that of the runs test.
\begin{algorithm}[t]
\caption{Active learning with automatic termination by testing convergence of $\cR$}
\label{alg1}
\begin{algorithmic}
    \State Sample $S_1=\{(x_1,y_1)\}$ and initialize $R=\{\}$
    \State Calculate GP posterior $q_1(f|S_1)$
    \For{$t=1,2,\ldots$}
        \State Sample by maximum uncertainty AL
        \State \qquad$x_{t+1}=\argmax_{x}\sigma(x,x)$
        \State Update dataset
        \State \qquad$S_{t+1} \leftarrow S_t \cup \{(x_{t+1},y_{t+1})\}$
        \State Calculate GP posterior $q_{t+1}(f|S_{t+1})$
        \State Update sequence of upper bounds
        \State \qquad$r_t \leftarrow \tilde{\cR}(q_t,q_{t+1})$, $\; R = R\cup\{r_t\}$
        \State Calculate median of $R$
        \State \qquad$m = {\rm median}(R)$
        \State Convert from upper bounds to binary
        \State Initialize $E=\{\}$
        \For{$i=1,2,\ldots,t$}
            \State $e_i \leftarrow {\rm sgn}(r_i-m)$
        \State $E = E \cup e_i$
        \EndFor
        \State Convergence test $E$ by using runs test
        \IIf{${\rm runTest}(E)$} ${\rm \mathbf{break}}$\EndIIf
    \EndFor
\end{algorithmic}
\end{algorithm}

\section{Experimental results}

This section describes the evaluation of the effectiveness of the proposed stopping criterion via a set of regression experiments with one artificial and five real-world datasets. The real-world datasets are obtained from the UCI machine learning repository. Every feature of these datasets is normalized so that their means are zero and their standard deviations are one\footnote{Simple Python implementation for our proposed method and competing methods have been submitted as the supplementary material and will be made publicly available after the review.}.

\subsection{Evaluation measure}
Let $q(f|S_T)$ be the posterior distribution of $f$ obtained using the complete dataset $S_T=\{(X_T,Y_T)\}$, and let $q(f|S_t)$ be the distribution obtained using the dataset $S_t$ of size $t$.
For quantitative evaluation of the determined stopping time, we define the {\it{optimal}} stopping time $t_{\rm{opt}}$ as the minimum data size $t$ that satisfies $ \Ex{q(f|S_t)}{\cL_\cD(f)}\leq \eta$, where $\eta$ is a predefined threshold. For determining the threshold $\eta$ for $t_{\rm opt}$, we use the complete dataset of interest and resampled $50$ points for training and $1950$ points for test for the artificial dataset, and $100$ points for training and the remaining points for test datasets for the real-world datasets $100$ times. By using these $100$ pairs of training and test datasets, we calculate the empirical estimate of the expected generalization errors, and $\eta$ is set to be the average $+$ 2 sd of the generalization errors. There are two possible approaches for active learning, aggressive and conservative~\citep{Bloodgood2009}. Since, basically, the aim of active learning is to save the cost for annotation, in this work we adopt aggressive approach and $\eta$ is set to be average $+ 2\mathrm{sd}$.

Because we consider GP regression, the loss function is defined as $l(f,x,y)=\frac{\beta}{2}(y-f(x))^2+\frac{1}{2}\log{(\beta/2\pi)}$. By denoting the posterior of $f$ by $q(f|S_t)=\mathcal{N}(\mu_t,\sigma_t)$ and the test dataset by $S_{\tilde T}=\{(x_i,y_i)\}^{\tilde T}_{i=1}$, the posterior average of the expected loss can be approximated by
\begin{align*}
    &\Ex{q(f|S_t)}{\cL_\cD(f)} \approx  \Ex{q(f|S_t)}{\cL_{S_{\tilde T}}(f)} \\
    &=\frac{\beta}{2{\tilde T}}\left\{\sum^{\tilde T}_{i=1}(y_i-\mu_t(x_i))^2+{\rm Tr}(\Sigma_t)\right\}+\frac{1}{2}\log{\frac{\beta}{2\pi}},
\end{align*}
With a stopping time of $t_{\ast}$ determined by a certain criterion, we consider
\begin{equation}
    e_{\rm stop}:=|t_\ast-t_{\rm opt}|
    \label{measure_stopping_criterion}
\end{equation}
as a measure of goodness for the stopping criterion.

\subsection{Dataset and methods for comparison}

We considered a simple one-dimensional model
\begin{align*}
    y_i =& e^{-(x_i-2)^2/2}+e^{-(x_i-6)^2/10}+(x^2_i+1)^{-1}+\epsilon_i
\end{align*}
with the additive Gaussian observation noise $p(\epsilon)=\mathcal{N}(0,\beta^{-1})$. From this generative model, we sampled $1,000$ pairs of inputs and outputs $(x_{i},y_i)$, where $x_i$ are uniform i.i.d. samples in $[-5,15]$. Among the $1,000$ pairs, $950$ pairs are retained as the test dataset and the remaining $50$ pairs are pooled for training the prediction model via active learning. Independent sampling of size $1,000$ is repeated $100$ times, and the average and standard deviation of $e_{{\rm{stop}}}$ defined in Eq.~\eqref{measure_stopping_criterion} are reported.

In the experiments, we compare the proposed criterion with the following four criteria:

(1) {\it{PAC-Bayesian criterion}}: The upper bound of the generalization error is approximated by using the conventional PAC-Bayesian result. By denoting the posterior of GP used in the training dataset $S_t$ by $q(f|S_t)$, the following upper bound can be derived~\citep{Alquier2016}:
\begin{align}
&\mathbb{E}_{q(f|S_t)}[\cL_\cD(f)] \nonumber \\
\leq& \mathbb{E}_{q(f|S_t)}[\cL_{S_t}(f)] +t^{-1}\KL{q(f|S_t)}{p(f)} \nonumber \\
&-t^{-1}\log{\delta}+(b-a)^2/2 := a_{t}.
\end{align}
To stabilize the KL divergence, $\kappa \mathbf{I}$ is added to the covariance matrices of the prior and the posterior, where $\kappa=0.01$. When $a_t$ is smaller than a prespecified threshold, we terminate the active learning procedure. The confidence parameter $\delta$ is set to $0.01$. $a$ and $b$ are set to $a=0$ and $b=\max_{y \in Y_T}y-\min_{y \in Y_T}y$. 

(2) {\it{Ground truth}}: We consider the convergence of $\cR(p(f),q(f|S_t))$, which is approximated by the test set $S_{\tilde{T}}$ and denoted by $\cR_{{\rm test}}(p(f),q(f|S_t))$. We stop learning when $\cR_{\rm test}(p(f),q(f|S_t))$ is larger than a certain prespecified threshold.

(3) {\it{Cross-validation criterion}}: We divide the samples collected in the active learning process into training and test datasets and evaluate the expected generalization error by $5$-fold cross validation. When the estimated generalization error is smaller than a predetermined threshold, the learning is stopped. 

(4) {\it{Maximum variance criterion}}: For classification problems, stopping criteria based on the uncertainty of class assignment have been proposed~\citep{Zhu2007,Zhu2008}. Herein, we consider a regression counterpart. Because we use the GP as a predictor, the posterior variance can be used as a measure of uncertainty. When the variance of all the possible or pooled data is smaller than a certain predefined threshold, learning is stopped.

\subsection{Parameter settings}
As a prior for the GP, we use a GP with a Gaussian kernel
$
    k(x,x') = \exp \left(-\frac{1}{2h^2}\|x-x'\|^2 \right).
$
The common parameters for all the methods are scale parameter $h$ for the kernel and variance of observation noise $\beta^{-1}$. These parameters are determined by using the marginal likelihood maximization using training datasets. The complete training dataset is not available in practice, but our aim is to set these parameter values in an objective manner and enable fair comparison.

In practical applications, it is better to update the hyperparameter at every update of the GP model, but the assumption in Corollary 1 does not hold when we change the hyperparameter. Marginalizing w.r.t. the hyperparameter is one of the reasonable approach, but again whether the assumption in Corollary 1 holds for the marginalized KL is uncertain. There are two possible practical approaches for this problem. (1) To assume that the Corollary 1 approximately holds between different hyperparameters and calculate Eq.~\eqref{proposed_upper_bound}. This would be reasonable because sample is added only one by one. (2) To calculate the upper bound by using a new hyper parameter in each time. In this work, we keep using the common parameters for the sake of simplicity.

For the proposed method, we must specify the significance level $\alpha$ for the statistical test. The level is fixed such that the type-I error rate is $0.1\%$. The KL divergence between GP posteriors and the range of the cost function $l$ are calculated in the same manner as for the {\it{PAC-Bayesian criterion}}. 
For the {\it{ground truth}}, for each dataset, we use the complete (training and test) dataset to perform bootstrap resampling $100$ times and evaluate $\cR_{\rm test}(p(f),q(f|S_t))$; the threshold for $\cR_{\rm test}(p(f),q(f|S_t))$ is set to the average $-$ 2$\times$sd of the bootstrap samples of $\cR_{\rm test}(p(f),q(f|S_t))$. 
The other three methods also require thresholds for termination. Because there is no universally applicable and objective method for setting the threshold, we considered one dataset, ``airfoil self-noise," as the reference dataset. In particular, we select the threshold minimizing Eq.~\eqref{measure_stopping_criterion} from a set of thresholds for each method. The set of thresholds is generated by sampling $10,000$ points at equal intervals within a range; the ranges is set to $[0.01,100]$, $[0.001,10]$, and $[0.0001,1]$ for the {\it{PAC-Bayesian criterion}}, {\it{cross-validation criterion}} and {\it{maximum variance criterion}}, respectively. 

It should be emphasized that the method of setting the threshold for the {\it{ground truth}} is not applicable in actual situations, which is the reason the method is called the {\it{ground truth}}. For the other three methods, because there exist no standard and objective threshold determination methods, we set the threshold values by using the reference dataset. In other words, these four criteria utilize the reference dataset, which is not available in practice, and the experimental setting is thus beneficial to them.

\begin{figure*}[t!]
    \centering
    \begin{tabular}{ccc}
    \includegraphics[width=5.4cm]{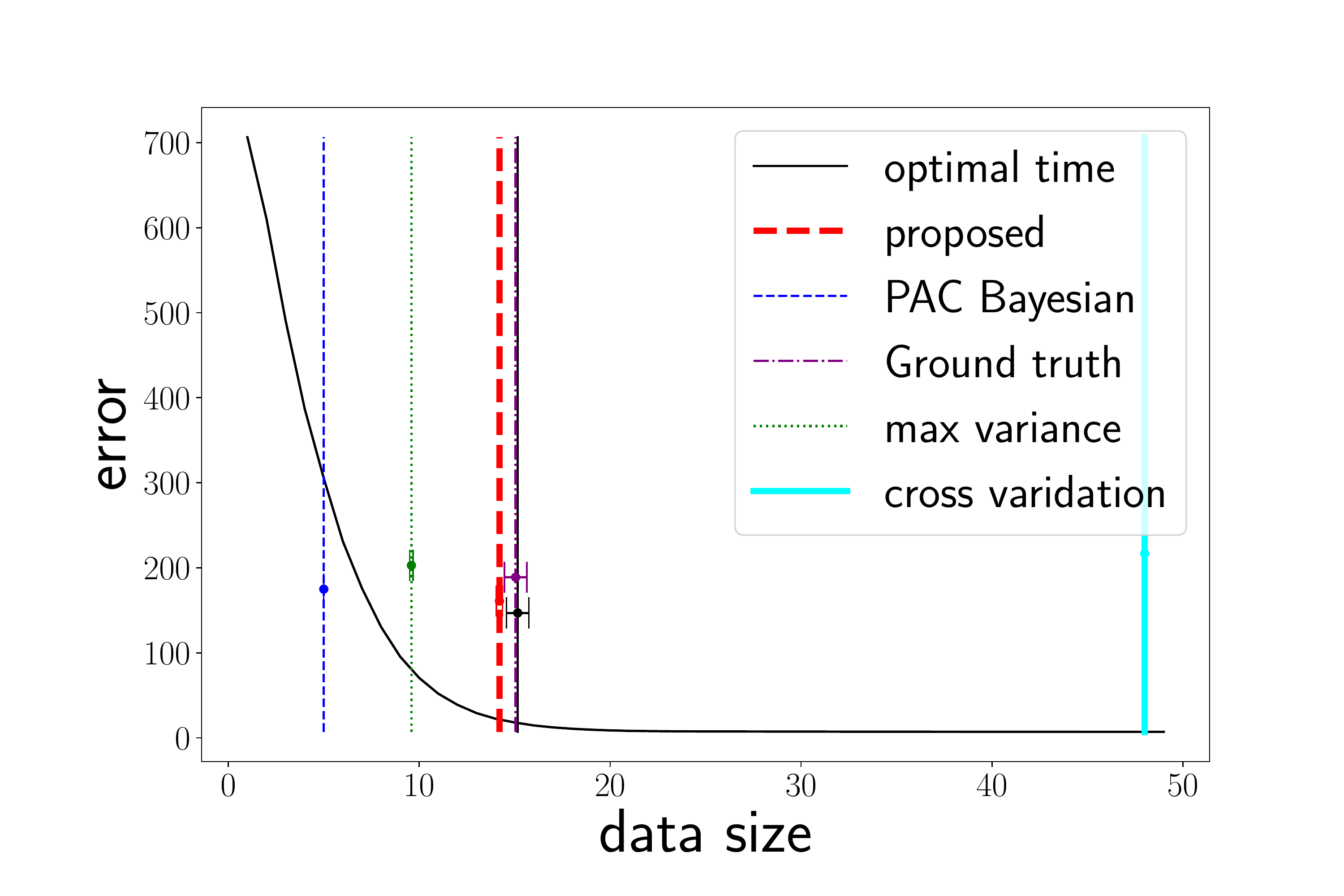}&
    \includegraphics[width=5.4cm]{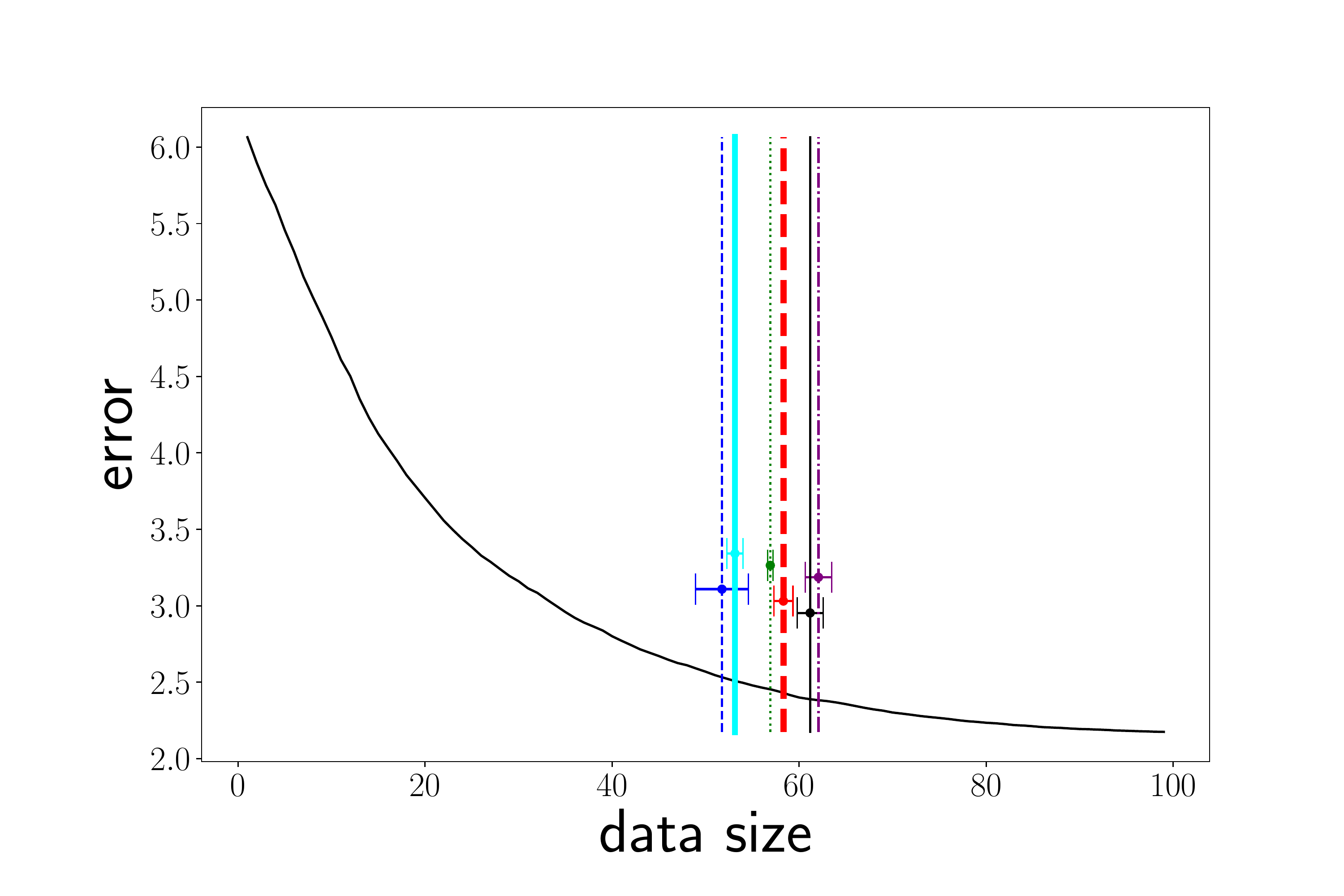}&
    \includegraphics[width=5.4cm]{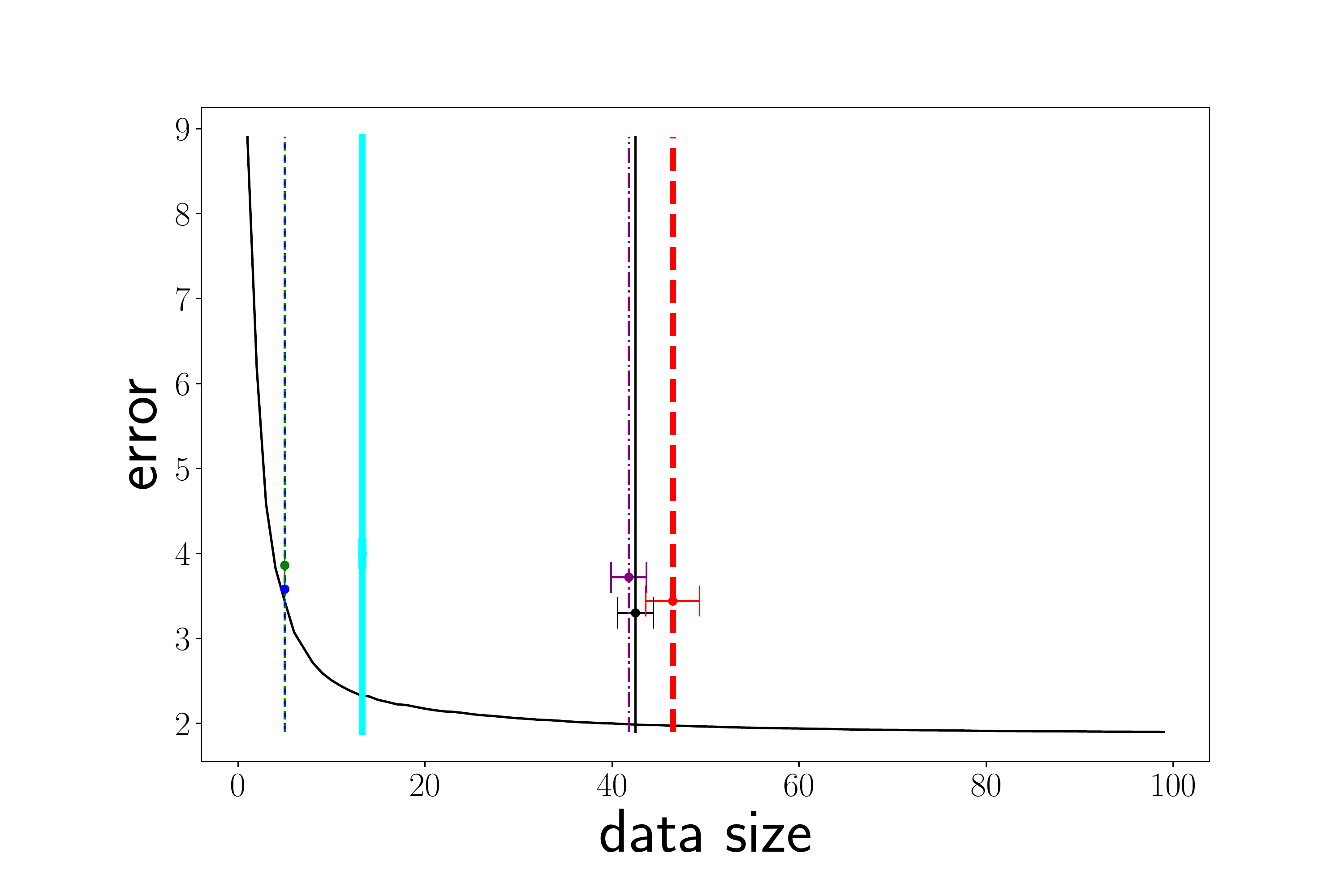} \\
    (a) artificial& (b) airfoil self-noise& (c) power plant \\
    \includegraphics[width=5.4cm]{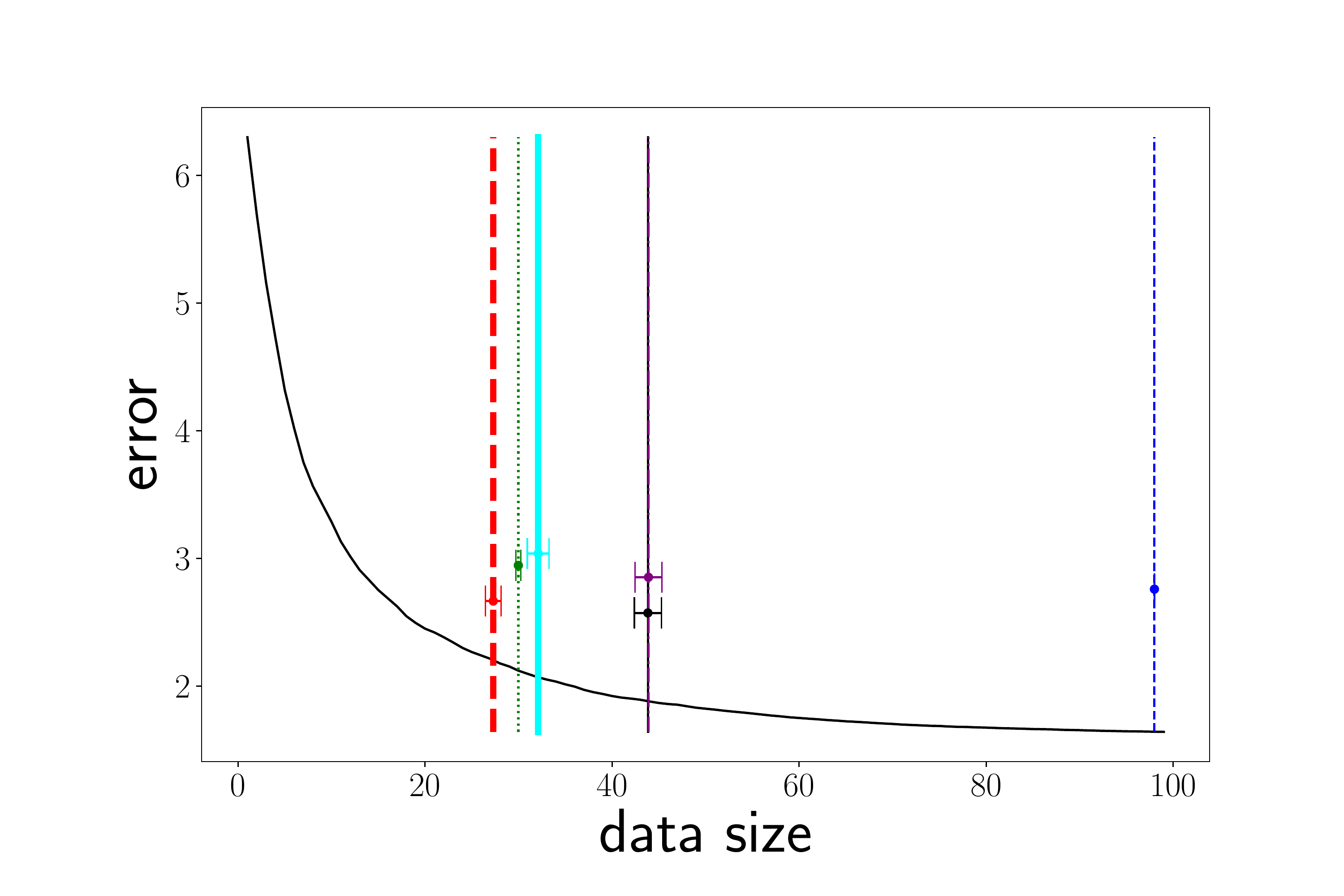}&
    \includegraphics[width=5.4cm]{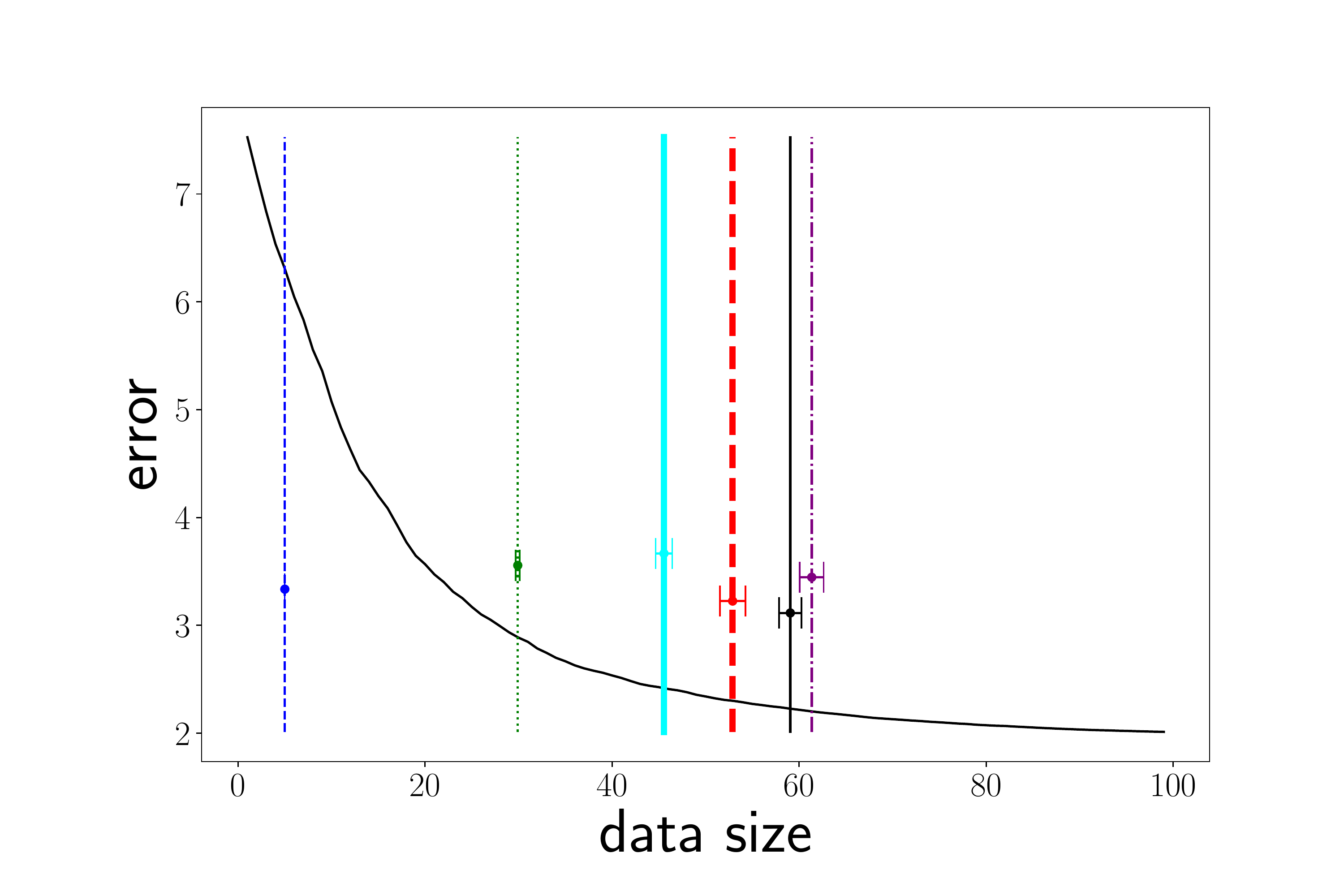}&
    \includegraphics[width=5.4cm]{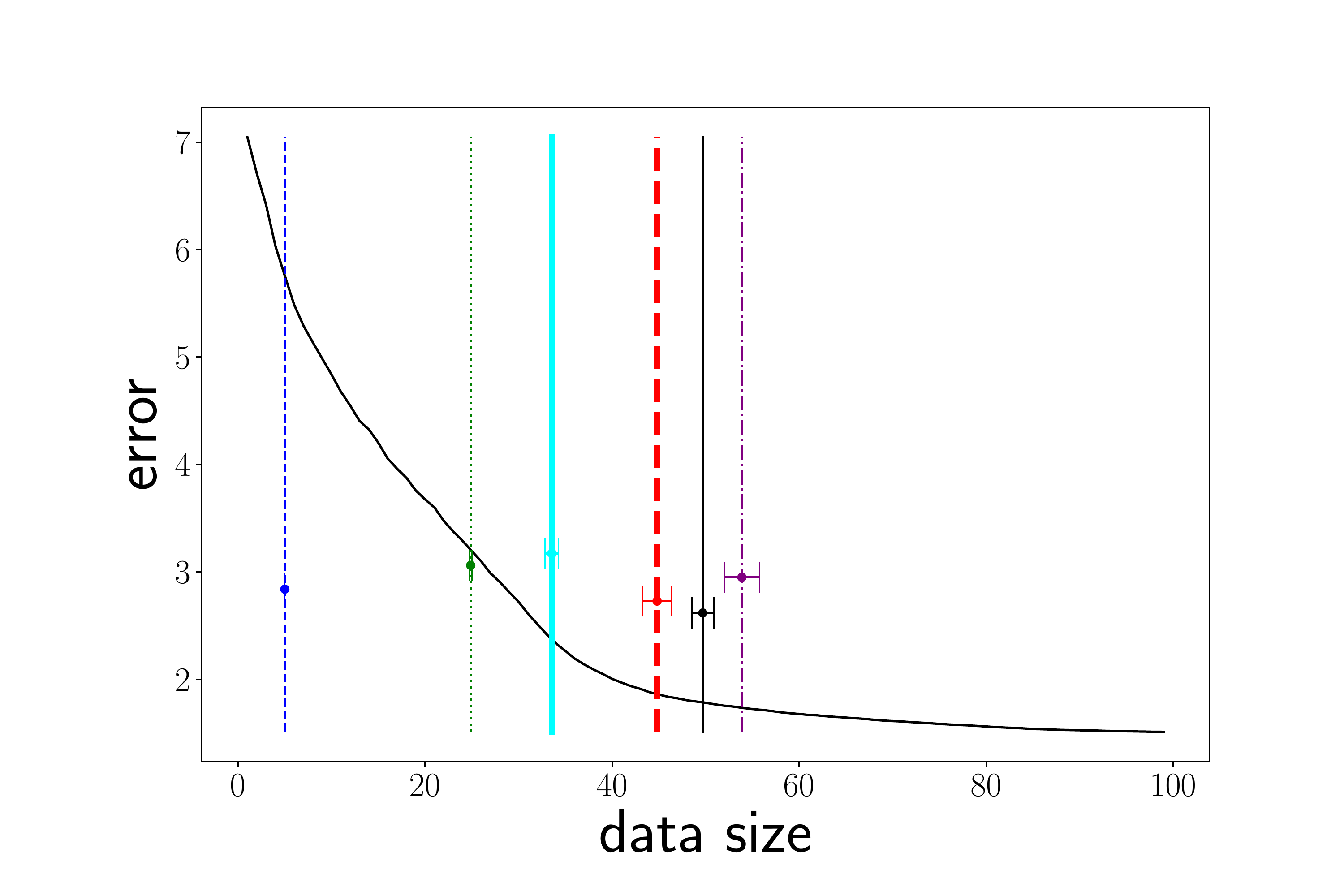} \\
    (d) protein& (e) concrete& (f) yacht
    \end{tabular}
    \caption{Averaged generalization error and training sample size. Vertical lines correspond to the optimal time and the stopping times determined by various methods.}
    \label{real_data_generalization_error}
\end{figure*}

\begin{table*}[thb]
  \centering
  \caption{Average and standard error of $e_{\rm stop}$.}
  \scalebox{0.8}{
  \begin{tabular}{l|c|c|c|c|c|c}
                    & artificial & airfoil self-noise & power plant &  protein & concrete & yacht \\ \hline \hline
    \# of samples/features
    & 2000/1 & 1503/6 & 10721/8 & 45730/9 & 1030/9 & 308/7  \\ \hline
    Ground truth & $0.25 \pm 0.05$ & $1.81 \pm 0.25$ & $2.11 \pm 0.39$ & $0.06 \pm 0.03$ & $4.29 \pm 0.58$ & $15.5 \pm 1.52$ \\ \hline
    Proposed & $\mathbf{2.38 \pm 0.58}$ & $13.52 \pm 1.05$ & $\mathbf{27.89 \pm 2.34}$ & $17.26 \pm 1.5$ & $\mathbf{15.83 \pm 1.28}$ & $\mathbf{16.33 \pm 1.28}$ \\
    PAC-Bayesian& $10.16 \pm 0.59$ & $30.35 \pm 1.8$ & $37.5 \pm 1.93$ & $54.17 \pm 1.44$ & $54.06 \pm 1.19$ & $44.7 \pm 1.18$ \\
    Cross validation & $32.84 \pm 0.59$ & $13.84 \pm 1.39$ & $29.2 \pm 1.92$ & $14.35 \pm 1.21$ & $17.09 \pm 1.35$ & $17.69 \pm 1.42$ \\
    Maximum variance & $5.57 \pm 0.61$ & $\mathbf{10.52 \pm 1.12}$ & $37.5 \pm 1.93$ & $\mathbf{14.26 \pm 1.39}$ & $29.15 \pm 1.25$ & $24.82 \pm 1.2$ \\
  \end{tabular}
  }
  \label{stopping_timing_accuracy_tbl_real_data}
\end{table*}

\subsection{Results}

The expected generalization error evaluated by the test dataset over the number of training samples is plotted in Fig.~\ref{stopping_timing_accuracy_tbl_real_data} for six datasets. 
From Fig.~\ref{real_data_generalization_error} (d), it can be noted that the stopping time determined by using the proposed method can be considerably different from $t_{\rm opt}$ in some cases, although the proposed method tends to terminate active learning when the training has been converged or is about to converge. The {\it{max variance criterion}} and {\it{cross validation criterion}} also offer stable and reasonable stopping times. The {\it{PAC-Bayesian criterion}} tends to either stop learning too early or overshoot the reasonable stopping time. 

Table~\ref{stopping_timing_accuracy_tbl_real_data} summarizes the values of the average and standard error of $e_{\rm stop}$ calculated in $100$ runs with independent resampling of the training and test datasets.
Except for the {\it{ground truth}}, the proposed method achieved the smallest average error in four cases over the six datasets. The experimental results indicate that the proposed criterion can accurately and stably determine when to stop active learning without using the test dataset for evaluating the stopping criterion or determining the threshold.

We will add reasoning on the experimental performance for each compared methods:
"PAC-Bayes" contains the training errors, and implicitly assumes the independence for training samples. Our bound does not contain training error term and remains valid in more general situation. We conjecture this is one of the reason why our method outperforms conventional PCA-Bayes. 
"CV" splits dataset so it uses less data for both training and validation. This would make the accuracy of the estimated test error low. Our method does not require a test data, leading to better performance. 
"Max variance" is directly connected to the acquisition function used for AL, while proposed method considers the divergence between posterior distributions of the predictive model. In certain ideal case, the proposed method could be considered as solely based on entropy, but before the convergence of the predictive model, both variance and mean largely affect the gap between before and after adding a new sample. This would be one of the reason why the max variance approach is nice but in many cases our method outperforms others.

\section{Conclusion}
We proposed a criterion for stopping active learning based on the PAC-Bayesian theory and a runs test. A noteworthy fact regarding the proposed criterion is that the gap between the expected generalization errors w.r.t. the posterior distributions before and after adding a new sample is deterministically bounded. The criterion does not require the test dataset for evaluating the generalization error and enables stopping of active learning automatically in a statistically reliable manner. Moreover, although we concentrated on GP regression in this study, the criterion can be used for both classification and regression problems with an arbitrary cost function. In the experiments, the effectiveness of the criterion was demonstrated in the cases of both an artificial dataset and real-world datasets.

Our newly derived upper bound does not assume independence of observation, and applicable to any objective function or any posterior distribution of the predictor. Therefore, the proposed bound will be used for other learning frameworks such as Bayesian optimization and online learning besides AL. The applicability of our new bound to other learning framework is one of the important future works.

\section*{Acknowledgement}
The authors express special thanks to the anonymous reviewers whose comments led to valuable improvements of this paper. Part of this work is supported by JST CREST JPMJCR1761. 

\bibliography{reference}
\bibliographystyle{unsrtnat}

\appendix
\section{Proof of Theorem 1 and Corollary 1}
We demonstrate the following three lemmas to prove Theorem 1 and Corollary 1.
\begin{lemma}\citep{Donsker1975,McAllester2003}
Let $\phi:\cF \rightarrow \mathbb{R}$ be any measurable function. Then, the following inequality holds:
\begin{equation}
    \Ex{p(f)}{\phi(f)} \leq \KL{p(f)}{p'(f)} + \log{\Ex{p'(f)}{e^{\phi(f)}}}.
    \label{change_of_measure}
\end{equation}
Here, $p$ and $p'$ are the probability distributions on $\cF$.
\label{lemma1}
\end{lemma}
\begin{lemma}\citep{Simic2008}
Let $h:X \rightarrow \mathbb{R}$ be a concave function, where $X\in[a,b]$. $p$ is a probability distribution with respect to $X$. We denote the difference of Jensen's inequality by $J(p,X)$, that is,
\begin{equation}
        J(p,X)=h(\Ex{p}{X})-\Ex{p}{h(X)}.
\end{equation}
Then, the following inequality holds:
\begin{equation}
    J(p,X) \leq 2h(\frac{a+b}{2})-h(a)-h(b).
\end{equation}
\label{lemma2}
\end{lemma}
\begin{lemma}

Let $q(f|S)$ and $q(f|S')$ be the posteriors with respect to $f$ given $S=(X,Y)$ and $S'=(X',Y')$, respectively. We assume that the prior of $q(f|S)$ is the same as that of $q(f|S')$. Then, the following inequality holds:
\begin{equation}
    \KL{q(f|S)}{q(f|S')} = \KL{q(\mathbf{f}_{X_+}|S))}{q(\mathbf{f}_{X_+}|S'))},
\end{equation}
\label{lemma3}
where $X_+:=X\cup X'$.
\end{lemma}
\begin{proof}
Let $X_\Omega$ be a universal set of input data. We denote $X_\Omega/X_+$ by $X_\ast$. Then, from the chain rule of KL divergence~\citep{Gray2011}, the following equation hold:
\begin{align}
    &\KL{q(f|S)}{q(f|S')} \\
    =&\KL{q(\mathbf{f}_{X_+}|S)}{q(\mathbf{f}_{X_+}|S')} \nonumber \\
    &+\Ex{q(\mathbf{f}_{X_+}|S)}{\KL{q(\mathbf{f}_{X_\ast}|\mathbf{f}_{X_+},S)}{q(\mathbf{f}_{X_\ast}|\mathbf{f}_{X_+},S')}}.
    \label{chain_rule}
\end{align}
We denote the prior of $q(\mathbf{f}_{X_\ast},\mathbf{f}_{X_+}|S)$ and $q(\mathbf{f}_{X_\ast},\mathbf{f}_{X_+}|S')$ by $p(\mathbf{f}_{X_\ast},\mathbf{f}_{X_+})$. Then, from the Bayesian theorem, the following equation holds:
\begin{align}
    q(\mathbf{f}_{X_\ast}|\mathbf{f}_{X_+},S)) =& \frac{p(\mathbf{f}_{X_\ast},\mathbf{f}_{X_+}|S)}{p(\mathbf{f}_{X_+}|S)} \\
    =&\frac{p(Y|\mathbf{f}_{X_+},X)p(\mathbf{f}_{X_\ast}|\mathbf{f}_{X_+})p(\mathbf{f}_{X_+})}{p(Y|X)} \nonumber \\
    &\times \frac{p(Y|X)}{p(Y|\mathbf{f}_{X_+},X)p(\mathbf{f}_{X_+})}  \\
    =&\frac{p(\mathbf{f}_{X_\ast},\mathbf{f}_{X_+})}{p(\mathbf{f}_{X_+})} =p(\mathbf{f}_{X_\ast}|\mathbf{f}_{X_+}).
\end{align}
Similarly, $q(\mathbf{f}_{X_\ast}|\mathbf{f}_{X_+},S')=p(\mathbf{f}_{X_\ast}|\mathbf{f}_{X_+})$ also holds. Therefore, if the prior of $q(f|S)$ is the same as that of $q(f|S')$, the second term of Eq.~\eqref{chain_rule} is zero.
\end{proof}
\textbf{Proof of Theorem 1}
\begin{proof}
By using Lemmas~\ref{lemma1} and ~\ref{lemma2}, the upper bound for $\cR(q(f|S),q(f|S'))$ is obtained as follows:
\begin{align}
    &\cR(q(f|S),q(f|S')) \\
    \leq& \KL{q(f|S)}{q(f|S')} \nonumber\\
    &+ \log{\Ex{q(f|S')}{e^{\cL_\cD(f)}}} -\Ex{q(f|S')}{\log{e^{\cL_\cD(f)}}} \label{proof_theorem1_1} \\
    \leq& \KL{q(f|S)}{q(f|S')} + 2\log{\frac{e^a+e^b}{2}}-a-b. \label{proof_theorem1_2}
\end{align}
By applying Lemma~\ref{lemma1} to Eq.~\eqref{change_of_measure}, we obtain Eq.~\eqref{proof_theorem1_1}.
Because the sum of the second and third terms of Eq.~\eqref{proof_theorem1_1} is the difference of Jensen's inequality, Lemma~\ref{lemma2} can be applied to it. Moreover, from $l\in[a,b]$, $\cL_\cD(f) \in [a,b]$ holds. Therefore, we obtain Eq.~\eqref{proof_theorem1_2}.
\end{proof}
\textbf{Proof of Corollary 1}
\begin{proof}
The proof is evident from Theorem 1 and Lemma~\ref{lemma3}.
\end{proof}

\section{Calculation of KL divergence between GPs}
\begin{lemma}
Let $q(f|S_t)$ and $q(f|S_{t+1})$ be the GP posteriors given $S_t=\{(x_i,y_i)\}^t_{i=1}$ and $S_{t+1}=\{(x_i,y_i)\}^{t+1}_{i=1}$, respectively. We assume that the prior of $q(f|S_t)$ is the same as that of $q(f|S_{t+1})$.
Let $\mu_t$ and $\sigma_t$ and $\beta$ be the mean and covariance functions of $q(f|S_t)$ and accuracy of Gaussian noise, respectively.Then the following equation holds:
\begin{align}
    &\KL{q(f|S_t)}{q(f|S_{t+1})} \nonumber \\
    =&\frac{1}{2}\beta \sigma_t(x_{t+1},x_{t+1})-\frac{1}{2}\log{(1+\beta \sigma_t(x_{t+1},x_{t+1}))} \nonumber \\
    &+\frac{1}{2}\frac{\beta \sigma_t(x_{t+1},x_{t+1})}{\sigma_t(x_{t+1},x_{t+1})+\beta^{-1}}(y_{t+1}-\mu_t(x_{t+1}))^2.
    \label{eq_gp_kl}
\end{align}
\label{lemma_gp_kl}
\end{lemma}
\begin{proof}

From Lemma~\ref{lemma3}, the following equation holds:
\begin{equation}
    \KL{q(f|S_t)}{q(f|S_{t+1})}=\KL{q(\mathbf{f}|S_t)}{q(\mathbf{f}|S_{t+1})},
\end{equation}
where $\mathbf{f}:=(f(x_1),f(x_2),\cdots,f(x_{t+1}))$.
When $S_{t+1}=(X_{t+1},Y_{t+1})$ is observed, $q(\mathbf{f}|S_{t+1})$ can be described as follows: 
\begin{align}
    q(\mathbf{f}|S_{t+1}) &= \frac{p(Y_{t+1}|\mathbf{f},X_{t+1})p(\mathbf{f})}{ p(Y_{t+1}|X_{t+1})} \nonumber \\
    &= \frac{p(y_{t+1}|\mathbf{f},x_{t+1})p(Y_t|\mathbf{f},X_t)p(\mathbf{f})}{\int p(y_{t+1}|\mathbf{f}',x_{t+1})p(Y_t|\mathbf{f}',X_t)p(\mathbf{f}') d\mathbf{f}'} \nonumber \\
    &= \frac{p(y_{t+1}|\mathbf{f},x_{t+1})p(Y_t|X_t)q(\mathbf{f}|S_t)}{\int p(y_{t+1}|\mathbf{f}',x_{t+1})p(Y_t|X_t)q(\mathbf{f}'|S_t) d\mathbf{f}'} \nonumber \\
    &= \frac{p(y_{t+1}|\mathbf{f},x_{t+1})q(\mathbf{f}|S_t)}{p(y_{t+1}|x_{t+1})}.
\end{align}
From this equation, $\KL{q(\mathbf{f}|S_{t-1})}{q(\mathbf{f}|S_t)}$ can be rewritten as follows:
\begin{align}
    &\KL{q(\mathbf{f}|S_t)}{q(\mathbf{f}|S_{t+1})} \nonumber \\
    =& \Ex{q(\mathbf{f}|S_t)}{\log\frac{q(\mathbf{f}|S_t)p(y_{t+1}|x_{t+1})}{p(y_{t+1}|\mathbf{f},x_{t+1})q(\mathbf{f}|S_t)}} \nonumber \\
    =& \log{p(y_{t+1}|x_{t+1})}-\Ex{q(\mathbf{f}|S_t)}{\log{p(y_{t+1}|\mathbf{f},x_{t+1})}} \nonumber \\
    =& \log{\int p(y_{t+1}|f_{t+1})q(f_{t+1}|S_t) df_{t+1}} \nonumber \\
    &-\int q(f_{t+1}|S_t)\log{p(y_{t+1}|f_{t+1})}df_{t+1},
    \label{eq_kl_expansion}
\end{align}
where $f_{t+1}:=f(x_{t+1})$. The first term of Eq.~\eqref{eq_kl_expansion} becomes logarithm of a normal distribution since $p(y_{t+1}|f_{t+1})$ and $q(f_{t+1}|S_t)$ are normal distributions. Specifically, from $p(y_{t+1}|f_{t+1})=\mathcal{N}(y_{t+1}|f_{t+1},\beta^{-1})$ and $p(f_{t+1}|S_t)=\mathcal{N}(f_{t+1}|\mu_t(x_{t+1}),\sigma_t(x_{t+1},x_{t+1}))$, the following equation holds:
\begin{align}
    &\log{\int p(y_{t+1}|f_{t+1})q(f_{t+1}|S_t) df_{t+1}} \nonumber \\ =&\log{\mathcal{N}(y_{t+1}|\mu_t(x_{t+1}),\sigma_t(x_{t+1},x_{t+1})+\beta^{-1})}
\end{align}

The second term can be rewritten as follows:
\begin{align}
    &-\int q(f_{t+1}|S_t)\log{p(y_{t+1}|f_{t+1})}df_{t+1} \nonumber \\ 
    =& \Ex{q(f_{t+1}|S_t)}{\frac{\beta}{2}(y_{t+1}-f_{t+1})^2}+\frac{1}{2}\log{2\pi\beta^{-1}} \nonumber \\
    =& {\frac{\beta}{2}\left(y^2_{t+1}-2y_{t+1}\mathbb{E}[f_{t+1}]+\mathbb{E}[f^2_{t+1}]\right)}+\frac{1}{2}\log{2\pi\beta^{-1}} \nonumber \\
    =& \frac{\beta}{2}(y_{t+1}-\mu_t(x_{t+1}))^2+\frac{\beta}{2}\sigma_t(x_{t+1},x_{t+1}) \nonumber \\
    &+\frac{1}{2}\log{2\pi\beta^{-1}}
\end{align}
From the above, the lemma is derived as follows:
\begin{align}
    &\KL{q(f|S_t)}{q(f|S_{t+1})} \nonumber \\
    =&-\frac{(y_{t+1}-\mu_t(x_{t+1}))^2}{2(\sigma_t(x_{t+1},x_{t+1})+\beta^{-1})} \nonumber\\
    &-\frac{1}{2}\log{2\pi(\sigma_t(x_{t+1},x_{t+1})+\beta^{-1})} \nonumber \\
    &+\frac{\beta}{2}(y_{t+1}-\mu_t(x_{t+1}))^2+\frac{\beta}{2}\sigma_t(x_{t+1},x_{t+1}) \nonumber \\
    &+\frac{1}{2}\log{2\pi\beta^{-1}} \nonumber \\
    =&\frac{1}{2}\beta \sigma_t(x_{t+1},x_{t+1})-\frac{1}{2}\log{(1+\beta \sigma_t(x_{t+1},x_{t+1}))} \nonumber \\
    &+\frac{1}{2}\frac{\beta \sigma_t(x_{t+1},x_{t+1})}{\sigma_t(x_{t+1},x_{t+1})+\beta^{-1}}(y_{t+1}-\mu_t(x_{t+1}))^2.
\end{align}
\end{proof}

\section{Tightness of the proposed upper bound}

\begin{figure}[!t]
    \centering
    \includegraphics[width=7cm]{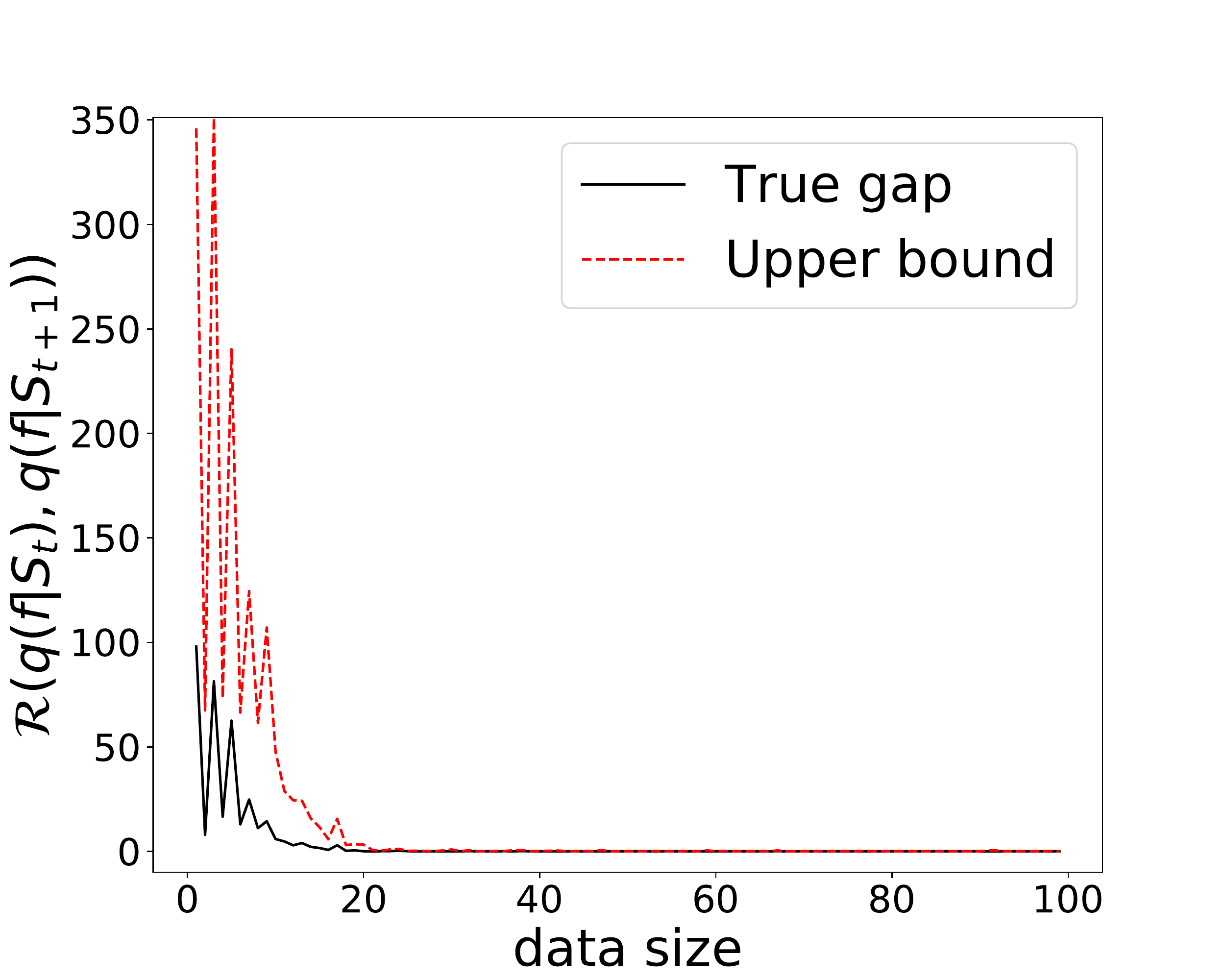} \\
    (a) regression \\
    \includegraphics[width=7cm]{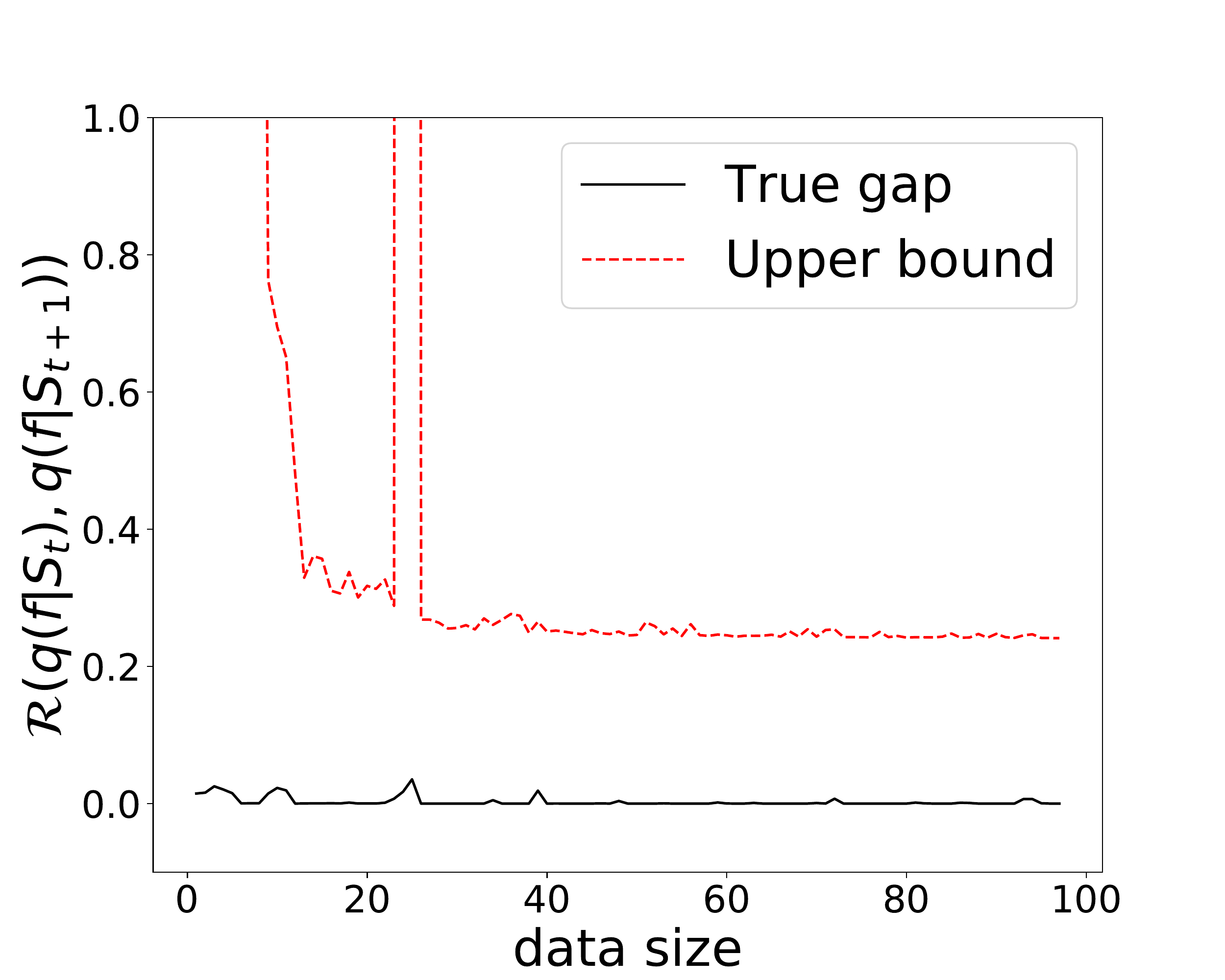} \\
    (b) classification \\
    \caption{Gap between the expected generalization errors, and its upper bound for (a) regression and (b) classification settings.}
    \label{tightness_of_upper_bound}
\end{figure}

We experimentally evaluated the tightness of the proposed upper bound of a gap between expected generalization errors before and after adding a new sample by comparing to the true gap approximated by using a large amount of test data. 

For the regression task, we used the artificial data used in the experiment of Section~5, while, for the classification task, we used the generated data $y_i ={\rm sgn}(\sin(2 \pi x_i))$, where ${\rm sgn(\cdot)}$ is the sign function. The kernel function and its hyperparameter are determined in the same manner explained in Section~5.

Figures~\ref{tightness_of_upper_bound} (a) and ~\ref{tightness_of_upper_bound} (b) show the gaps between the expected generalization errors and their upper bounds for (a) regression and (b) classification tasks. We see that (1) increasing the data size leads to a tight upper bound in both cases of regression and classification, and the KL-divergence term converges to zero. Moreover, (2) the bound could be trivial when the KL divergence takes a large value ($>1$), particularly in classification setting. Also, as the KL-divergence is always non-negative, when the gap of the expected generalization error is negative, the bound is meaningless. As can be seen from Fig.~\ref{tightness_of_upper_bound} (a), the bound works well in a regression setting and the offers reasonable tightness. 

\end{document}